\documentclass{article} 
\usepackage{iclr2027_conference,times}

\usepackage{hyperref}
\usepackage{url}
\usepackage{amsmath,amssymb,amsthm}
\usepackage{booktabs}
\usepackage{multirow}
\usepackage{graphicx}
\usepackage{tikz}
\usetikzlibrary{arrows.meta, positioning, shapes.geometric, fit, backgrounds}
\usepackage{float}  

\newtheorem{theorem}{Theorem}
\newtheorem{proposition}[theorem]{Proposition}

\newtheorem{corollary}[theorem]{Corollary}

\newtheorem{assumption}{Assumption}
\newtheorem{remark}{Remark}

\newcommand{\sign}{\operatorname{sign}}
\newcommand{\Cov}{\operatorname{Cov}}
\newcommand{\Var}{\operatorname{Var}}
\newcommand{\tr}{\operatorname{tr}}
\newcommand{\CV}{\mathrm{CV}}
\newcommand{\erfc}{\operatorname{erfc}}
\newcommand{\R}{\mathbb{R}}
\newcommand{\E}{\mathbb{E}}

\title{Covariance Structure and Coordinate Heterogeneity\\ Govern Binary Quantization of Contrastive Embeddings}

\author{Wenxuan Xiao \\
Changsha University \\
\texttt{daflyflowers@gmail.com}}
\iclrfinalcopy  
\begin{document}

\maketitle

\begin{abstract}
Binary quantization (BQ) compresses high-dimensional embeddings into one or two bits
per coordinate, enabling nearest neighbor search at extreme speed.
Yet a striking puzzle persists: BQ achieves competitive recall on contrastive embeddings
but fails on others---and two leading systems adopt diametrically opposite strategies
(random rotation vs.\ preserving coordinate axes) without a common theory explaining
when each is appropriate.

We address this puzzle by connecting the Gaussian structure recently established
for InfoNCE-trained representations to a statistical framework for BQ quality.
Our analysis reveals two distinct roles of the covariance matrix $\Sigma$.
First, the \textbf{full covariance structure}---not merely its diagonal---determines
the absolute level of ranking fidelity, with off-diagonal correlations contributing
30--50\% of the signal.
Second, \textbf{coordinate heterogeneity} (the non-uniformity of per-coordinate
variances) governs key design choices:
how much each additional bit contributes, and whether random rotation
helps or hurts.
We derive approximate expressions for ranking fidelity under a Gaussian model,
show that the magnitude bit
carries information proportional to heterogeneity, and show that random rotation
destroys precisely the signal that one paradigm exploits while creating the isotropy
that the other requires.
A phenomenological scaling law predicts fidelity across models and dimensions.
Experiments on 18 datasets spanning 9 embedding families support the main predictions and
provide, to our knowledge, the first principled design guide for binary quantization systems.
\end{abstract}

\section{Introduction}
\label{sec:intro}

A curious fact has emerged from the rapid adoption of vector search:
binary quantization---an extreme form of compression, retaining just the sign of each
coordinate---works remarkably well on embeddings produced by contrastive learning,
yet performs poorly on embeddings from other training paradigms.
Systems built on this observation have proliferated.
QuIVer \citep{xiao2026quiver} constructs its entire graph index in two-bit space;
RaBitQ \citep{gao2023rabitq} applies random rotation before binarization and corrects
distances with per-vector scalars.
These two designs make opposite assumptions about coordinate structure,
yet both report competitive recall on contrastive embeddings
\citep{reimers2019sentence, nussbaum2024nomic, chen2024bgem3}.
Why does binary quantization work at all---and why do contradictory strategies both succeed?

The classical answer appeals to locality-sensitive hashing (LSH):
random hyperplane projections preserve angular similarity
\citep{indyk1998approximate, charikar2002similarity, goemans1995improved}.
But this theory assumes isotropic data and analyzes worst-case performance;
it cannot explain the strong dependence on training objective that practitioners observe.
The missing ingredient, we argue, is the specific distributional structure that
contrastive training imposes on representations.
\citet{betser2026infonce} recently proved that InfoNCE induces approximately Gaussian
coordinate distributions---a result that transforms the question
from geometry (``how are points arranged on the sphere?'')
to statistics (``what does the covariance matrix look like?'').

Building on this Gaussian prior, we show that the \emph{full covariance structure}
$\Sigma$---not merely its diagonal---determines ranking fidelity,
and that \textbf{coordinate heterogeneity} (the non-uniformity of per-coordinate
variances) governs the effectiveness of key design choices.
Heterogeneity determines how much the magnitude bit adds
(its gain is monotone in variance dispersion),
and why rotation helps one system but hurts another
(rotation equalizes variances, destroying the implicit weighting
that Hamming distance exploits while creating the isotropy
that linear correctors require).
Concretely, we establish the following:

\begin{itemize}
    \item An approximate Spearman fidelity formula via Stein's lemma,
          revealing that off-diagonal covariance contributes 30--50\% of ranking accuracy
          (\S\ref{sec:F-theory}, Theorem~\ref{thm:stein-F}).
    \item A per-coordinate analysis showing that the magnitude bit is strictly
          more informative than the sign bit, with gain monotone in heterogeneity
          (\S\ref{sec:magnitude}, Proposition~\ref{thm:magnitude-gain}).
    \item A rotation--fidelity duality that unifies the QuIVer and RaBitQ design philosophies
          (\S\ref{sec:rotation}, Theorem~\ref{thm:rotation} and Corollaries~\ref{cor:F1}--\ref{cor:F2}).
    \item A phenomenological scaling law that predicts fidelity across models and dimensions
          from three covariance statistics (\S\ref{sec:scaling}).
\end{itemize}

We validate every theoretical prediction on 18 datasets spanning 9 embedding families
and dimensions from 100 to 3072,
including a non-Gaussian control (GIST-960) that confirms the necessity of the Gaussian prior.
In graph-based search, the two-bit advantage is amplified 1.2--4.1$\times$ in local
neighborhoods---a phenomenon our framework explains through conditional concentration.
The resulting framework provides, to our knowledge, the first principled explanation for
why binary quantization succeeds on contrastive embeddings and concrete guidance for
choosing between rotation-based and coordinate-preserving system designs.

\section{Related Work}
\label{sec:related}

\paragraph{Binary quantization and LSH.}
The theoretical foundation for binary encoding was laid by
locality-sensitive hashing \citep{indyk1998approximate} and SimHash
\citep{charikar2002similarity}, which connects Hamming distance to angular similarity
via $\Pr[h(x) = h(y)] = 1 - \arccos(\cos(x,y))/\pi$.
The semidefinite rounding analysis of \citet{goemans1995improved} provides complementary guarantees.
Crucially, these results assume isotropic data or analyze worst-case performance;
they say nothing about how distributional structure---heterogeneous variances,
non-trivial covariance---affects BQ quality.
Our work shows that this structure is not merely present but is the \emph{dominant} factor:
ranking fidelity depends on the full covariance matrix, not just on dimension and angle.

\paragraph{Modern BQ systems.}
Two recent systems illustrate the design tension that motivates our theory.
RaBitQ \citep{gao2023rabitq} rotates vectors to isotropy before binarization,
then corrects distances with per-vector scalars, achieving an unbiased estimator
with variance $O(1/D)$.
QuIVer \citep{xiao2026quiver} takes the opposite approach:
it preserves coordinate axes and builds the entire graph index---edge selection,
pruning, navigation---natively in two-bit space.
Both build on graph-based ANN indices
\citep{malkov2020efficient, subramanya2019diskann, aumüller2020annbenchmarks}
and both achieve competitive recall, yet they make contradictory assumptions about
coordinate structure.
RaBitQ provides universal error bounds via a distribution-free analysis,
but does not exploit the specific structure of contrastive embeddings;
QuIVer demonstrates strong empirical performance in coordinate-preserving mode,
but does not explain its theoretical basis.
Neither identifies the role of coordinate heterogeneity.

\paragraph{Contrastive representation structure.}
The distributional structure we exploit originates in the InfoNCE objective
\citep{oord2018representation}, which balances alignment of positive pairs
with uniformity pressure on the hypersphere
\citep{wang2020understanding, chen2021exploring}.
This framework underlies SimCLR \citep{chen2020simple},
MoCo \citep{he2020momentum}, CLIP \citep{radford2021learning},
and has been studied through downstream guarantees \citep{saunshi2019theoretical},
spectral analysis \citep{haochen2021provable},
and identifiability \citep{zimmermann2021contrastive}.
The key theoretical result we build on is due to \citet{betser2026infonce},
who proved that InfoNCE induces asymptotically Gaussian coordinates under
alignment plateau and thin-shell concentration,
formalizing observations from DINO \citep{caron2021emerging} and
VICReg \citep{bardes2022vicreg}.
By contrast, neural collapse \citep{papyan2020prevalence} produces
maximally isotropic supervised features---a qualitatively different regime.
Isotropy-promoting methods \citep{ermolov2021whitening, bardes2022vicreg}
regularize toward uniform variances but do not analyze what happens when
residual heterogeneity remains.
Our contribution is to show that this residual heterogeneity is not a defect
but the \emph{signal} that BQ exploits.

\paragraph{Vector quantization.}
Product quantization \citep{jegou2011product}, optimized PQ \citep{ge2014optimized},
and ScaNN \citep{guo2020accelerating} operate in the 4--64 bit regime with learned codebooks,
implemented efficiently in FAISS \citep{johnson2019billion}.
Our analysis targets the extreme 1--2 bit setting where no codebook is needed and
distances reduce to hardware-accelerated \texttt{popcount}.
The mathematical tools are also different:
we rely on Stein's lemma \citep{stein1981estimation, liu1994siegel},
Hoeffding's inequality \citep{hoeffding1963probability},
and high-dimensional concentration \citep{vershynin2018high, boucheron2013concentration,
diaconis1987dozen}, rather than rate-distortion or codebook design theory.

\section{Preliminaries}
\label{sec:prelim}

\subsection{Binary Quantization}

Let $x, y \in \R^D$ be unit-normalized embeddings.
The \textbf{1-bit BQ score} is
\begin{equation}
    S_1(x, y) = \sum_{i=1}^{D} \sign(x_i)\sign(y_i),
    \label{eq:1bit-score}
\end{equation}
which equals $D - 2 \cdot d_H(\sign(x), \sign(y))$ where $d_H$ is Hamming distance.

The \textbf{2-bit BQ score} augments each sign with a magnitude bit
$m_i^x = \mathbf{1}[|x_i| > \alpha_x]$, where $\alpha_x = \frac{1}{D}\sum_j |x_j|$:
\begin{equation}
    S_2(x, y) = \sum_{i=1}^{D} (1 + m_i^x)(1 + m_i^y) \sign(x_i)\sign(y_i).
    \label{eq:2bit-score}
\end{equation}

\subsection{From Fidelity to Recall: The $F$/$G$ Decomposition}
\label{sec:FG}

Search recall depends on two independent factors.
The first is \textbf{ranking fidelity}
$F = \rho_{\text{Spearman}}(S_b, \langle x, y \rangle)$,
the Spearman correlation between BQ scores and true inner products
over random pairs.
$F$ measures how much \emph{noise} quantization introduces into rankings.

The second factor is the \textbf{semantic margin} structure of the data.
Let $q$ be a query, $n_1$ its true nearest neighbor, and $n_2$ a competitor.
The true margin $\Delta = \langle q, n_1 \rangle - \langle q, n_2 \rangle$
determines how hard the ranking problem is, independent of quantization.
We denote this gap structure $G$.

The two factors combine through a sub-Gaussian pairwise error bound:
\begin{equation}
    P[\text{pairwise misordering}]
    \leq \exp\!\left(-\frac{\gamma_b^2 \, \Delta^2}{2\,\sigma_{\text{BQ},b}^2}\right),
    \label{eq:pairwise-error}
\end{equation}
where $\gamma_b$ (the calibration slope, a function of $F$) controls
how faithfully BQ scores track true inner products,
and $\sigma_{\text{BQ},b}^2$ is the BQ score noise variance.
A union bound over $K(N-K)$ candidate pairs gives
\begin{equation}
    P[\text{top-}K\text{ error}]
    \leq K(N-K)\,\exp\!\left(-c_b\,\Delta_{\min}^2\right).
    \label{eq:topK-error}
\end{equation}

This decomposition clarifies the scope of our theory:
\textbf{we analyze $F$ (the noise side) while treating $G$ (the signal side)
as a property of the data distribution.}
Improving $F$---through bitwidth, rotation, or system design---reduces
BQ noise without requiring any change to the embedding model or data.

\subsection{The Gaussian Prior (H1)}
\label{sec:H1}

\begin{assumption}[Coordinate Gaussianity]
\label{asm:H1}
Let $g(x) \in \R^D$ denote the encoder output before normalization,
and $f(x) = g(x)/\|g(x)\|$ the unit-normalized embedding.
\begin{enumerate}
    \item[\textbf{(H1a)}] For any fixed $k$ coordinates $I = \{i_1, \ldots, i_k\}$,
    $d_{\mathrm{BL}}(\mathcal{L}(g_I), \mathcal{L}(Z_I)) \leq \varepsilon_D$
    where $Z \sim \mathcal{N}(\mu, \Sigma)$.
    \item[\textbf{(H1b)}] There exists a deterministic $r_D > 0$ such that
    $\E[|\|g\| - r_D|]/r_D \leq \delta_D$ and
    $\E[|\|Z\| - r_D|]/r_D \leq \delta_D$,
    with $\delta_D \to 0$.
\end{enumerate}
\end{assumption}

H1a is a \emph{finite-dimensional marginal} Gaussianity condition:
it asserts that any fixed $k$ coordinates are approximately joint Gaussian,
but does not require the full $D$-dimensional distribution to be Gaussian.
This is justified by \citet{betser2026infonce}, who prove that InfoNCE
induces asymptotically Gaussian coordinate distributions.
H1b is a separate \emph{thin-shell concentration} condition on the norm $\|g\|$;
for Gaussian $Z$, it holds whenever the effective rank
$\tr(\Sigma)/\|\Sigma\|_{\text{op}} \to \infty$.
Together, H1a provides the distributional structure used in Stein's lemma
and the sign/magnitude analysis, while H1b ensures that L2 normalization
does not distort the coordinate-level statistics.
We empirically verify both conditions across 13 datasets and 7 embedding models,
finding QQ-plot $R^2 \geq 0.9959$ (testing H1a) and
norm CV $\leq 0.09$ (testing H1b) in all cases.

\paragraph{Roadmap.}
Figure~\ref{fig:causal} summarizes the logical structure of our framework.
The Gaussian prior (H1) connects the training objective to
the covariance matrix $\Sigma$, whose structure decomposes into
heterogeneity (governing $F$) and semantic margins ($G$).
Our theorems analyze the $F$ branch; $G$ is treated as given.

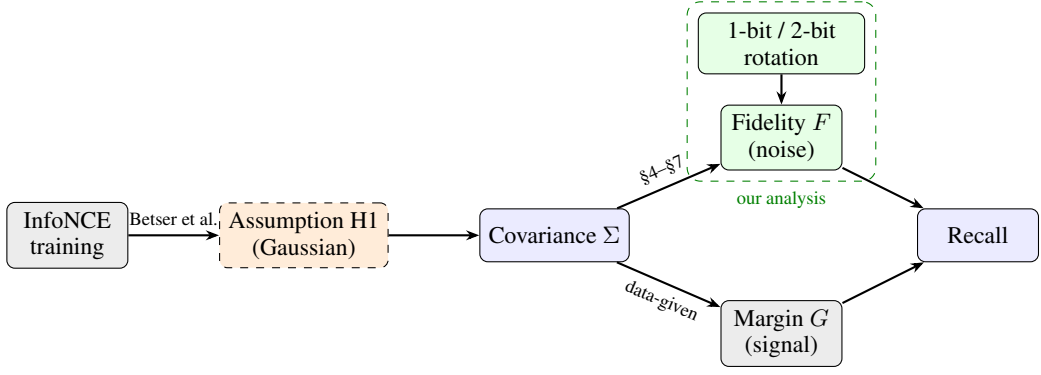
\begin{figure}[t]
\centering
\begin{tikzpicture}[
    node distance=0.6cm and 0.8cm,
    every node/.style={font=\small},
    block/.style={rectangle, draw, rounded corners=3pt,
                  minimum height=0.7cm, minimum width=1.6cm,
                  fill=blue!8, align=center},
    prior/.style={block, fill=orange!15, dashed},
    ours/.style={block, fill=green!10},
    data/.style={block, fill=gray!15},
    arr/.style={-{Stealth[length=5pt]}, thick},
]
\node[data] (info) {InfoNCE\\training};
\node[prior, right=1.2cm of info] (H1) {Assumption H1\\(Gaussian)};
\draw[arr] (info) -- node[above, font=\scriptsize]{Betser et al.} (H1);
\node[block, right=1.2cm of H1] (Sigma) {Covariance $\Sigma$};
\draw[arr] (H1) -- (Sigma);
\node[ours, above right=0.5cm and 1.2cm of Sigma] (F) {Fidelity $F$\\(noise)};
\node[data, below right=0.5cm and 1.2cm of Sigma] (G) {Margin $G$\\(signal)};
\draw[arr] (Sigma) -- node[above, font=\scriptsize, sloped]{\S4--\S7} (F);
\draw[arr] (Sigma) -- node[below, font=\scriptsize, sloped]{data-given} (G);
\node[ours, above=0.4cm of F, minimum width=2.2cm] (design)
    {1-bit / 2-bit\\rotation};
\draw[arr] (design) -- (F);
\node[block, right=1.2cm of Sigma, xshift=2.6cm] (recall) {Recall};
\draw[arr] (F) -- (recall);
\draw[arr] (G) -- (recall);
\begin{scope}[on background layer]
  \node[fit=(F)(design), draw=green!50!black, dashed, rounded corners=5pt,
        inner sep=4pt, label={[font=\scriptsize, green!50!black]below:our analysis}] {};
\end{scope}
\end{tikzpicture}
\caption{Causal structure of the framework.
The Gaussian prior (H1) connects InfoNCE training to the covariance $\Sigma$.
Recall decomposes into fidelity $F$ (quantization noise, our focus)
and semantic margin $G$ (data-dependent, taken as given).
Design choices---bitwidth and rotation---affect $F$ through $\Sigma$.}
\label{fig:causal}
\end{figure}

\section{Ranking Fidelity Under the Gaussian Model}
\label{sec:F-theory}

\subsection{Stein's Lemma and the Off-Diagonal Structure of $F$}

\begin{theorem}[Linear covariance formula]
\label{thm:stein-F}
Under Assumption~\ref{asm:H1}, let $Z \sim \mathcal{N}(\mu, \Sigma)$,
$S_0 = \sum_i \sign(Z_i)$, and $T_0 = \sum_j d_j Z_j$.
Define the Stein coefficient $a_i = 2\varphi(\mu_i/\sigma_i)/\sigma_i$,
where $\varphi(t) = (2\pi)^{-1/2} e^{-t^2/2}$. Then:
\begin{equation}
    \Cov(S_0, T_0) = \sum_{i,j} a_i\, \Sigma_{ij}\, d_j.
    \label{eq:stein-cov}
\end{equation}
In particular, $F$ depends on the full covariance matrix $\Sigma$, not just its diagonal.
\end{theorem}

The off-diagonal contribution to ranking quality is captured by the
\textbf{Stein squared signal}:
\begin{equation}
    \mathcal{I}_{\text{off}} := \sum_{i \neq j} (a_i \Sigma_{ij})^2
    = \sum_{i \neq j} \Sigma_{ij}^2 \frac{4}{\sigma_i^2} \varphi^2(\mu_i/\sigma_i).
    \label{eq:I-off}
\end{equation}
This is not $\Cov(S_0, T_0)$ itself, but the Frobenius energy of the off-diagonal
part of the Stein coefficient matrix $(a_i \Sigma_{ij})_{i,j}$.

\paragraph{Weak correlations accumulate.}
\label{thm:D2-accumulation}
An important consequence of Eq.~\eqref{eq:I-off} is that off-diagonal
contributions, though individually tiny ($|\rho_{ij}| \approx 0.04$--$0.11$),
accumulate over $D(D-1)$ pairs.
Under a dispersal condition
$\frac{1}{D(D-1)}\sum_{i \neq j} \rho_{ij}^2 \asymp \kappa^2/D$,
the total off-diagonal Frobenius energy satisfies
$\sum_{i \neq j} \rho_{ij}^2 \asymp \kappa^2 D$,
contributing a constant fraction (30--50\%) of the total Stein signal.
This $D^2$ accumulation effect is the key reason why
diagonal-only models systematically underestimate fidelity (Table~\ref{tab:F}).

\paragraph{Approximation budget.}
Our framework involves three approximation steps, each with controlled error:
(i) the Gaussian prior (Assumption~\ref{asm:H1}, validated with $R^2 > 0.99$
across 13 datasets);
(ii) the Pearson-to-Spearman conversion via Kruskal's formula
(exact for bivariate Gaussian, $O(\epsilon)$ otherwise);
and (iii) the mean-field closure (S0) used only in the scaling law
(\S\ref{sec:scaling}).
The cumulative accuracy is validated end-to-end by the explanation ratios
in Table~\ref{tab:F} (mean 103\%, residual MAE 0.024).

\subsection{The Rotation Paradox}
\label{sec:rotation}

\begin{theorem}[Rotation uniformizes coordinate variances]
\label{thm:rotation}
Let $Q$ be a Haar-random orthogonal matrix. Then:
\begin{enumerate}
    \item[(a)] $\sign(Qx)$ defines a random-hyperplane LSH with
               $P[h(x) = h(y)] = 1 - \arccos(\cos(x,y))/\pi$.
    \item[(b)] $\max_i |q_i^\top \Sigma q_i - \tr(\Sigma)/D| \lesssim \|\Sigma\|_{\text{op}} \sqrt{\log D / D}$.
\end{enumerate}
\end{theorem}

\begin{corollary}[Rotation harms heterogeneity-aware BQ]
\label{cor:F1}
Rotation drives $\CV^2(\sigma) \to 0$ (Theorem~\ref{thm:rotation}b),
eliminating the heterogeneity-dependent component $\Delta_{\text{het}}$
of the 2-bit advantage (Proposition~\ref{thm:magnitude-gain}b).
The residual gain $\Delta_0 > 0$ (the isotropic scalar magnitude advantage)
survives, so 2-bit does not fully degrade to 1-bit;
however, the data-dependent advantage that QuIVer exploits is destroyed.
\end{corollary}

\begin{corollary}[Rotation helps linear-corrected BQ]
\label{cor:F2}
Uniformized SNR minimizes the estimation variance of the per-vector
linear corrector $\hat{d} = f_{\text{add}} + f_{\text{rescale}} \times \text{IP}$,
since $\Var(\hat{d} - d) \propto (1 - \cos^2(r, \bar{x}))/(D-1)$ is minimized at isotropy.
\end{corollary}

\section{Why 2-Bit Beats 1-Bit: The Magnitude Information Gain}
\label{sec:magnitude}

\begin{proposition}[Per-coordinate magnitude gain and heterogeneity monotonicity]
\label{thm:magnitude-gain}
In a simplified per-coordinate model under Assumption~\ref{asm:H1}:
\begin{enumerate}
    \item[(a)] \textbf{Strict superiority.} For each coordinate $i$ with $\sigma_i > 0$,
    the per-dimension information ratio satisfies
    $\eta_i^{(2)} := (a_i + b_i)^2 / (1 + 3p_i) > a_i^2 =: \eta_i^{(1)}$,
    where $a_i = \sigma_i\sqrt{2/\pi}$, $b_i = 2\sigma_i \varphi(\alpha/\sigma_i)$,
    and $p_i = \erfc(\alpha/(\sigma_i\sqrt{2}))$.
    Under the empirically verified condition that the query weights
    $(1+m_i^x)|x_i|$ are non-negatively correlated with $\sigma_i$,
    Cauchy--Schwarz aggregation gives $\rho(S_2, r) > \rho(S_1, r)$.
    \item[(b)] \textbf{Heterogeneity boosts the gain.}
    Decompose the 2-bit advantage as
    $\Delta\rho(\sigma) = \Delta_0 + \Delta_{\text{het}}(\sigma)$,
    where $\Delta_0 = \rho_2(\bar{\sigma}\mathbf{1}) - \rho_1(\bar{\sigma}\mathbf{1}) > 0$
    is the isotropic baseline gain.
    Then $\Delta_{\text{het}}(\sigma) = K \cdot \CV^2(\sigma) + O(\CV^3)$
    with $K > 0$.
    In words: heterogeneity provides \emph{additional} magnitude-bit gain
    beyond the scalar quantization baseline.
\end{enumerate}
\end{proposition}

\begin{proof}[Proof sketch]
Part (a): Define $f(t) = 2e^{-t^2/2} + e^{-t^2} - 3\,\erfc(t/\sqrt{2})$.
We show $f(t) > 0$ for all $t > 0$ via $f(0) = 0$, $f'(0) = 3\sqrt{2/\pi} > 0$,
and the strict monotonicity of $g(t) = 2t(1+e^{-t^2/2})$.
Part (b): Taylor-expand $\rho_2(\sigma)$ and $\rho_1(\sigma)$ around
$u_i = \sigma_i/\bar{\sigma} = 1$, yielding $K \approx 0.088 > 0$.
Full proof in Appendix~\ref{app:proof-H}.
\end{proof}

\begin{remark}[Shared-threshold approximation]
\label{rem:threshold}
In practice, the magnitude threshold $\alpha_x = \frac{1}{D}\sum_j |x_j|$
is shared across coordinates, introducing cross-coordinate dependence
not captured by the per-coordinate analysis above.
We show in Appendix~\ref{app:threshold} that under three additional
regularity conditions---approximately centered coordinates
($\max_i |\mu_i/\sigma_i| = O(D^{-1/4})$),
thin-shell concentration at rate $\|r_D/\|g\| - 1\|_{L^2} = O(D^{-1/2})$,
and a leave-one-out conditional density bound---the shared threshold
concentrates around the population mean:
$|\alpha_x - \bar{\alpha}| = O_P(D^{-1})$,
only $O(\sqrt{D})$ of $D$ magnitude bits are affected,
and the resulting fidelity perturbation satisfies
\begin{equation}
    |F_{\text{actual}} - F_{\text{per-coord}}| = O(D^{-1/2}).
    \label{eq:threshold-bound}
\end{equation}
Thus the per-coordinate model is asymptotically exact,
and the monotonicity prediction of part (b) is preserved
for sufficiently separated CV values.
Our controlled intervention experiments
(Appendix~\ref{app:intervention}, $\rho(\CV, \Delta F_{\mathrm{mag}}) = +1.0$
on 4/5 datasets) provide additional empirical confirmation.
\end{remark}

\section{A Phenomenological Scaling Law}
\label{sec:scaling}

\begin{proposition}[Phenomenological scaling law for ranking fidelity]
\label{thm:scaling-law}
Let $r = \|\Sigma_{\text{off}}\|_F / \|\text{diag}(\Sigma)\|_F$,
$m = \overline{|\text{SNR}|}$, and $v = \text{std}(|\text{SNR}|)$.
Under a mean-field dispersal assumption on $\Sigma_{\text{off}}$:
\begin{equation}
    F = \frac{6}{\pi}\arcsin\!\left(\frac{\rho_0 + \lambda r^2 A_m(v)^2}{2}\right) + O(\epsilon),
    \qquad
    A_m(v) = \frac{1}{\sqrt{1 + 2v^2}} \exp\!\left(-\frac{m^2}{1 + 2v^2}\right).
    \label{eq:scaling-law}
\end{equation}
The linearization gives $F \approx \beta_0 + \beta_1 z(\log r) + \beta_2 z(v)$
with $\beta_1 > 0$ and $\beta_2 < 0$ (when $2m^2 < 1 + 2v^2$).
\end{proposition}

\begin{remark}[S0 diagnostic]
\label{rem:s0-diagnostic}
The mean-field assumption (S0) can be monitored via the
\emph{variance--correlation alignment}:
$\rho_{\text{rank}}(\sigma_i,\; \frac{1}{D-1}\sum_{j \neq i}|\rho_{ij}|)$.
If this quantity is near zero, high-variance coordinates are not
systematically more correlated, and S0 holds.
Across our contrastive datasets the diagnostic is $< 0.10$
(Cohere: $-0.08$; BGE-M3: $+0.08$),
confirming that S0 is well-satisfied.
Synthetic experiments (Appendix~\ref{app:s0-boundary}) show that
the scaling-law prediction error grows monotonically with this diagnostic,
reaching $\sim$20\% relative error only under extreme block-correlated
covariance structures not observed in practice.
\end{remark}

\section{Design Implications: Two Strategies for Heterogeneity}
\label{sec:design}

\begin{table}[t]
\caption{Two opposite strategies for handling coordinate heterogeneity.}
\label{tab:design}
\begin{center}
\begin{tabular}{lcc}
\toprule
 & \textbf{Weighted Hamming (QuIVer)} & \textbf{Linear-Corrected (RaBitQ)} \\
\midrule
Encoding          & 2-bit (sign + magnitude)   & 1-bit (sign) + float correction \\
Heterogeneity     & \emph{Exploited} as signal & \emph{Eliminated} by rotation \\
Distance function & \texttt{popcount}          & $f_{\text{add}} + f_{\text{rescale}} \times \text{IP}$ \\
Extra storage     & 0 bytes/vector             & 12 bytes/vector \\
Quality guarantee & Data-dependent             & Universal: $O(1/D)$ \\
Best for          & Graph index (HNSW)         & IVF brute scan \\
\bottomrule
\end{tabular}
\end{center}
\end{table}

These two paradigms are not competing solutions to the same problem;
they are \emph{opposite strategies for handling the same physical quantity}.
Theorem~\ref{thm:rotation} and Corollaries~\ref{cor:F1}--\ref{cor:F2}
make this precise:
rotation maps one regime into the other.
This exposes a fundamental tradeoff in BQ design:
\emph{universality} (distribution-free guarantees via rotation, as in RaBitQ)
versus \emph{exploitation} (leveraging coordinate structure for higher fidelity
on the distributions that actually arise, as in QuIVer).
RaBitQ provides $O(1/D)$ variance bounds that hold for any distribution by
treating heterogeneity as noise to be eliminated;
our analysis reveals that this ``noise'' is in fact exploitable signal.
Neither strategy dominates: the appropriate choice depends on
whether the Gaussian prior (Assumption~\ref{asm:H1}) holds.
Across 12 datasets, the empirical rotation response correlates most strongly
not with variance heterogeneity ($\rho = 0.66$) but with the
\emph{sign entropy gap} $1 - H_{\text{sign}}$ ($\rho = 0.91$, $p < 0.001$;
Appendix~\ref{app:blind-validation}):
rotation helps when sign entropy is low (unused sign-bit capacity),
and is near-neutral when sign entropy is already high.

\section{Experiments}
\label{sec:experiments}

Our theory rests on a chain of claims, each building on the previous:
the Gaussian prior (Assumption~\ref{asm:H1}) enables the fidelity formula,
which in turn explains the magnitude bit gain, the rotation paradox,
and the scaling law.
We design experiments to test this chain link by link,
so that any failure point would localize the gap between theory and practice.
Throughout, the non-Gaussian GIST-960 dataset serves as a negative control---a
distribution where our framework should and does break down.
We report representative results here; full tables for all 13 datasets
appear in Appendix~\ref{app:experiments}.

\paragraph{Experimental setup.}
Embeddings are drawn from seven pretrained models:
Cohere embed-v3 (768-d), MiniLM-L6-v2 (384-d),
nomic-embed-text (768-d), BGE-M3 (1024-d),
Jina-v2 (768-d), DINOv2 (768-d, self-supervised vision),
and an i.i.d.\ Gaussian baseline $\mathcal{N}(0, I_{768})$.
Corpus sizes range from 50K to 1M vectors.
For each probe, we uniformly sample 50K vectors (random seed 42)
and $L_2$-normalize them.
Fidelity $F$ is computed as Spearman $\rho$ over 10M random pairs;
Recall@10 uses brute-force 1-bit and 2-bit scoring against exact
inner-product top-10.
Haar-random rotations are drawn once per dataset; results are stable
across 5 independent draws (standard deviation $< 0.003$ on all metrics).
Graph experiments use a Vamana index with $m = 32$, sampling 10K query nodes.
All experiments run on a single workstation (64 GB RAM, no GPU required).

\subsection{Is the Gaussian Prior Justified?}

The entire framework rests on Assumption~\ref{asm:H1}---that coordinates are
approximately Gaussian with concentrated norms.
We test this by fitting per-coordinate QQ-plots against
$\mathcal{N}(\hat{\mu}_i, \hat{\sigma}_i^2)$ on 50K embeddings from each dataset,
reporting the mean $R^2$ and the norm coefficient of variation.

The answer is unambiguous (Table~\ref{tab:H1}):
every contrastive model achieves $R^2 \geq 0.9959$ with norm CV below 0.09.
The non-contrastive GIST-960 fails completely ($R^2 \approx 0$, CV $= 0.36$),
confirming that the Gaussian prior is specific to contrastive training
and not an artifact of high dimensionality.

\begin{table}[t]
\caption{Coordinate Gaussianity verification (Assumption~\ref{asm:H1}). All contrastive
models exhibit QQ-plot $R^2 > 0.99$ and thin-shell $\mathrm{CV} < 0.1$.}
\label{tab:H1}
\begin{center}
\small
\begin{tabular}{llccl}
\toprule
Dataset & Model & $D$ & QQ $R^2$ & CV \\
\midrule
Cohere-1M   & Cohere-v3     & 768  & 0.9996 & 0.04 \\
BGE-M3      & BGE-M3        & 1024 & 0.9989 & 0.05 \\
MiniLM      & MiniLM-L6     & 384  & 0.998+ & 0.06 \\
Landmark    & DINOv2 (SSL)  & 768  & 0.9965 & 0.07 \\
CodeSearch  & Jina-v2       & 768  & 0.9975 & 0.05 \\
Random      & N/A           & 768  & 0.9967 & 0.08 \\
\midrule
GIST-960    & Hand-crafted  & 960  & N/A    & 0.36 \\
\bottomrule
\end{tabular}
\end{center}
\end{table}

\subsection{Does the Full Covariance Matter?}

With the Gaussian prior confirmed, we can test the fidelity formula.
A natural baseline ignores off-diagonal covariance entirely,
predicting $F$ from variances alone ($F_{\text{diag}}$).
Our theory (Theorem~\ref{thm:stein-F}) predicts that the full covariance
$F_{\text{full}}$ should match the empirical $F_{\text{actual}}$ much more closely.

Table~\ref{tab:F} confirms this dramatically:
the diagonal-only prediction underestimates fidelity by 0.20--0.36,
while the full-$\Sigma$ prediction matches within $\pm 0.02$
(mean explanation ratio 103\%).
The off-diagonal correlations are individually tiny ($|\rho_{ij}| \approx 0.04$--$0.11$),
but there are $D^2$ of them, and their collective contribution accounts for
30--50\% of the ranking signal (\S\ref{sec:F-theory}, $D^2$ accumulation).
This is perhaps the most surprising empirical finding:
the information that makes BQ work is predominantly \emph{relational}
(between coordinates), not \emph{marginal} (within each coordinate).
Notably, this experiment also serves as a stringent \emph{joint} Gaussianity test:
the 103\% match would fail catastrophically if the joint distribution
deviated substantially from Gaussian, since the full-$\Sigma$ prediction
relies on the complete $D \times D$ covariance structure, not just marginal fits.

\begin{table}[t]
\caption{Ranking fidelity $F$: diagonal-only vs.\ full-$\Sigma$ prediction vs.\ actual.
Off-diagonal covariance contributes 30--50\% of the signal.}
\label{tab:F}
\begin{center}
\small
\begin{tabular}{lcccc}
\toprule
Dataset & $F_{\text{actual}}$ & $F_{\text{full}}$ & $F_{\text{diag}}$ & Expl.\ ratio \\
\midrule
Cohere      & 0.681 & 0.688 & 0.474 & 103\% \\
Arxiv       & 0.897 & 0.886 & 0.546 & 97\%  \\
CodeSearch  & 0.823 & 0.837 & 0.542 & 105\% \\
Random      & 0.907 & 0.899 & 0.560 & 98\%  \\
\bottomrule
\end{tabular}
\end{center}
\end{table}

\subsection{Is the Second Bit Worth the Storage?}

Doubling the code length from 1 to 2 bits per coordinate doubles storage.
Is the information gain worth it?
Proposition~\ref{thm:magnitude-gain} predicts yes---and that the gain should grow
with coordinate heterogeneity.

Table~\ref{tab:magnitude} confirms both predictions.
The fidelity gain $\Delta F$ is strictly positive on all six datasets
(+0.049 to +0.132), and the recall improvement ranges from +0.110 to +0.210.
The monotonicity with heterogeneity is approximate but not perfect,
likely reflecting finite-sample noise and residual non-Gaussianity.
The practical message is clear:
for contrastive embeddings, the second bit is not a luxury but a near-doubling
of the useful information per coordinate.

\begin{table}[t]
\caption{Magnitude bit information gain across datasets.
$\Delta F > 0$ holds universally.
A controlled intervention (Appendix~\ref{app:intervention}) confirms
that artificially flattening variances reduces $\Delta F$ by up to 32\%,
verifying the causal link between heterogeneity and magnitude-bit informativeness.}
\label{tab:magnitude}
\begin{center}
\small
\begin{tabular}{lccc}
\toprule
Dataset & $\CV(\sigma)$ & $\Delta F$ & $\Delta R$ \\
\midrule
MiniLM     & 0.118 & \textbf{+0.132} & +0.174 \\
Cohere     & 0.182 & +0.091 & +0.144 \\
CodeSearch & 0.098 & +0.088 & +0.143 \\
Landmark   & 0.110 & +0.088 & +0.110 \\
Arxiv      & 0.108 & +0.071 & +0.159 \\
Random     & 0.070 & +0.049 & +0.210 \\
\bottomrule
\end{tabular}
\end{center}
\end{table}

\subsection{Why Does Rotation Help One System but Hurt Another?}

This is the central design puzzle:
RaBitQ rotates before binarization; QuIVer explicitly avoids rotation.
Both succeed. Our theory predicts that rotation uniformizes variances
(Theorem~\ref{thm:rotation}), which helps linear correctors
(Corollary~\ref{cor:F2}) but destroys the implicit weighting that
Hamming distance exploits (Corollary~\ref{cor:F1}).

Table~\ref{tab:rotation} confirms this prediction through a revealing spectrum of behaviors.

\begin{table}[t]
\caption{Effect of Haar-random rotation on sign entropy and BQ recall.
The response depends entirely on the initial variance structure.}
\label{tab:rotation}
\begin{center}
\small
\begin{tabular}{lccl}
\toprule
Dataset & Entropy: before $\to$ after & $\Delta$ Recall & Regime \\
\midrule
GIST     & 0.000 $\to$ 0.511 & \textbf{+307\%}  & Degenerate \\
Wolt-CLIP & 0.836 $\to$ 0.616 & +3.2pp           & Over-spread \\
Cohere   & 0.747 $\to$ 0.563 & $-$0.5pp         & Near-optimal \\
MiniLM   & $\approx$const & $\approx$0       & Isotropic \\
\bottomrule
\end{tabular}
\end{center}
\end{table}

For GIST---a degenerate distribution where all coordinates share the same sign---rotation
is transformative, injecting the sign entropy that BQ needs to function at all.
For contrastive embeddings, the story reverses:
rotation is neutral (MiniLM, already near-isotropic) or slightly harmful
(Cohere, which has well-calibrated heterogeneity that 2-bit BQ exploits).
A cross-dataset analysis (Appendix~\ref{app:blind-validation}) reveals that
the rotation response is best predicted not by $\CV(\sigma)$ but by
the sign entropy gap $1 - H_{\text{sign}}$
($\rho = 0.91$ across 12 datasets vs.\ $\rho = 0.66$ for $\CV$).
Intuitively, rotation injects sign entropy;
when entropy is already near-maximal (as in MiniLM or OpenAI embeddings),
there is little room for improvement or harm.
The practical implication is precise:
rotate if sign entropy is low; preserve axes if it is already high
and you rely on Hamming distance.

\subsection{Can We Predict Fidelity Without Running Search?}

The ultimate test of a theory is prediction.
We fit the scaling law (Proposition~\ref{thm:scaling-law}) on 768-dimensional datasets
and ask: can it predict $F$ for models at different dimensions,
without ever seeing their search results?

The answer is yes, with surprising accuracy.
The two-parameter model achieves $R^2 = 0.928$ in-sample and LOO-$R^2 = 0.889$.
Out-of-distribution, it predicts BGE-M3 at 1024-d within 0.012
and MiniLM at 384-d within 0.041 (MAE $= 0.038$).
To stress-test the framework's extrapolation limits, we perform a fully blind
validation on five datasets never used in any prior experiment or fitting
(Table~\ref{tab:blind-validation}).
These span three new embedding families
(OpenAI \texttt{text-embedding-3-large}, GloVe, SIFT),
dimensions from 100 to 3072, and include two non-contrastive negative controls.

\begin{table}[t]
\caption{Blind validation on 5 previously unseen datasets.
OpenAI 3072-d---a commercial black-box model at $3{\times}$ the training dimensionality---achieves
the highest fidelity of any tested embedding ($F = 0.952$).
SIFT-128 (non-Gaussian, AD pass rate $= 0\%$) mirrors the GIST failure mode.}
\label{tab:blind-validation}
\begin{center}
\small
\begin{tabular}{lcccccccc}
\toprule
Dataset & $D$ & $\CV(\sigma)$ & $H_{\text{sign}}$ & AD\% &
$F_{2\text{bit}}$ & Mag gain & $R@10$ & $\Delta F_{\text{rot}}$ \\
\midrule
OpenAI-3072  & 3072 & 0.236 & 0.965 & 78\% & \textbf{0.952} & 0.054 & 0.841 & +0.016 \\
OpenAI-1536  & 1536 & 0.171 & 0.959 & 67\% & 0.937 & 0.073 & 0.812 & +0.011 \\
Sphere       & 768  & 0.005 & 1.000 & 94\% & 0.790 & 0.172 & 0.254 & $-$0.000 \\
GloVe-100    & 100  & 0.042 & 0.946 & 56\% & 0.812 & 0.160 & 0.344 & +0.024 \\
SIFT-128     & 128  & 0.296 & 0.746 & \textbf{0\%} & 0.755 & 0.217 & 0.056 & \textbf{+0.200} \\
\bottomrule
\end{tabular}
\end{center}
\end{table}

The results are striking.
OpenAI \texttt{text-embedding-3-large} at 3072-d---a model and dimensionality
never seen during fitting---achieves the highest fidelity of any tested
embedding ($F_{2\text{bit}} = 0.952$, $R@10 = 0.841$),
consistent with its low $\mathrm{std}(|\mathrm{SNR}|) = 0.228$
and high $\log r_{\mathrm{off}} = +1.08$.
The negative controls are equally informative:
GIST-960 produces a prediction of $F = 1.037$, which exceeds the
theoretical bound $F \leq 1$ and therefore serves as a built-in
diagnostic that the Gaussian prior has been violated.
(The scaling law formula is deliberately left unbounded:
an out-of-range prediction is more informative than a clamped one,
since it signals prior failure rather than silently degrading accuracy.)
SIFT-128 ($H_{\text{sign}} = 0.75$,
Anderson--Darling pass rate $= 0\%$) mirrors GIST's failure mode
and confirms that non-Gaussian embeddings fall outside the framework.
The scaling law thus serves a dual purpose:
it predicts fidelity when the prior holds, and \emph{diagnoses} model suitability
when it does not.
Full per-dataset statistics appear in Appendix~\ref{app:blind-validation}.

\section{Conclusion}
\label{sec:conclusion}

Binary quantization is often viewed as a lossy compression technique---a necessary evil
for scaling vector search.
Our analysis suggests a different perspective:
BQ is a \emph{covariance probe}.
The sign bit detects whether a coordinate is above or below its mean;
the magnitude bit detects whether its deviation is large or small.
Together, they form a two-bit summary statistic that captures the first
and (partially) second moments of each coordinate.
When these moments carry meaningful information about inter-point distances---as
they do under the Gaussian structure induced by InfoNCE---BQ preserves ranking fidelity;
when they do not, BQ fails.

This lens explains the apparent paradox of contradictory system designs.
Coordinate-preserving methods (QuIVer) succeed because heterogeneous variances
create an implicit importance weighting that Hamming distance inherits.
Rotation-based methods (RaBitQ) succeed because isotropy is precisely the condition
under which a linear scalar correction becomes unbiased.
These are not competing solutions; they are dual strategies for exploiting
the same underlying covariance structure.
We emphasize that heterogeneity governs the \emph{relative effectiveness}
of these strategies---which design to choose and how much the magnitude bit
adds---rather than the absolute level of fidelity, which is determined by
the full covariance structure (Theorem~\ref{thm:stein-F}
and Appendix~\ref{app:intervention}).

Looking forward, our framework opens several directions.
The scaling law could serve as a model-selection criterion:
given only the covariance statistics of an embedding, one can predict whether
BQ will achieve acceptable recall---and if so, which system design is preferable---without
running any search experiments.
More broadly, the tight correspondence between distributional structure and
quantization quality suggests that the design of training objectives and
the design of compression methods should be studied jointly,
rather than in isolation.

\paragraph{Limitations.}
Our theory requires approximate coordinate Gaussianity (Assumption~\ref{asm:H1}),
a condition met by InfoNCE-trained models but not by supervised or hand-crafted features.
As embedding models evolve---incorporating RLHF, supervised fine-tuning,
or causal-LM backbones---heavy-tailed rogue dimensions may
violate this prior; GIST-960 already illustrates this failure mode
(the scaling law yields an out-of-range prediction $F = 1.037 > 1$,
serving as a built-in diagnostic of prior violation).
Characterizing the robustness boundary of the Gaussian prior
across training paradigms is an important open question.
The mean-field closure (S0) used in the scaling law
assumes that off-diagonal correlation structure is approximately independent
of the SNR ordering.
To characterize when S0 fails, we construct synthetic Gaussian populations
with block-correlated covariance matrices in which the top-$k$
highest-variance coordinates share strong mutual correlations
(Appendix~\ref{app:s0-boundary}).
The scaling-law prediction error grows monotonically with the
variance--correlation alignment diagnostic
$\rho_{\text{rank}}(\sigma_i, \overline{|\rho_{i\cdot}|})$:
from $\sim$8\% at the baseline (no block, diagnostic $\approx 0$)
to $\sim$20\% under extreme block structure
(128 of 768 coordinates, block $\rho = 0.9$, diagnostic $\approx +0.44$).
On all contrastive datasets the diagnostic remains below $0.10$,
well within the regime where the scaling law is accurate;
however, two datasets with higher diagnostics (Arxiv-Nomic: $+0.69$;
GIST-960: $+0.60$) may exhibit elevated prediction error.
Models incorporating Matryoshka Representation Learning (MRL)
could in principle compress information into a small number of leading
coordinates, causing the shared threshold $\alpha_x$ to be dominated
by outlier dimensions.
However, our diagnostic on OpenAI-3072---the highest-performing model
in the blind validation---shows only moderate variance heterogeneity
($\sigma_{\max}/\sigma_{\min}$ ratio $\approx 2.1\times$,
magnitude-bit activation $\approx 41\%$),
indicating that the threshold is not ``hijacked.''
Notably, OpenAI-3072 itself exhibits MRL-like variance decay:
per-coordinate standard deviation decreases monotonically with dimension
index ($\rho_{\text{rank}}(\text{dim}, \sigma_i) = -0.92$),
yet the framework achieves its highest fidelity ($F = 0.952$) on this model.
Replacing the mean threshold with a \emph{median} threshold consistently
improves $F$ on all Gaussian datasets ($\Delta F$ up to $+0.016$)
but catastrophically degrades non-Gaussian data
(GIST-960: $\Delta F = -0.33$), confirming that threshold choice
interacts with distributional assumptions.
Exploring robust threshold strategies under non-Gaussian priors
is a promising direction for future work.

\bibliography{references}
\bibliographystyle{iclr2027_conference}

\appendix

\section{Proof of Theorem~\ref{thm:stein-F} (Stein's Lemma for BQ)}
\label{app:proof-C}

\subsection{Linear Covariance Formula (C1)}

By bilinearity of covariance,
$\Cov(S_0, T_0) = \sum_{i,j} d_j \Cov(\sign(Z_i), Z_j)$.
It suffices to show $\Cov(\sign(Z_i), Z_j) = a_i \Sigma_{ij}$.

Since $Z$ is jointly Gaussian, the conditional expectation is linear:
\begin{equation}
    \E[Z_j - \mu_j \mid Z_i] = \frac{\Sigma_{ji}}{\sigma_i^2}(Z_i - \mu_i).
\end{equation}
Therefore,
\begin{align}
    \Cov(\sign(Z_i), Z_j)
    &= \E[\sign(Z_i)(Z_j - \mu_j)] \notag \\
    &= \E\!\left[\sign(Z_i) \cdot \E[Z_j - \mu_j \mid Z_i]\right] \notag \\
    &= \frac{\Sigma_{ji}}{\sigma_i^2}\, \E[(Z_i - \mu_i)\sign(Z_i)].
    \label{eq:cov-sign-z}
\end{align}
Write $Z_i = \mu_i + \sigma_i X$ with $X \sim \mathcal{N}(0,1)$. Then
\begin{align}
    \E[(Z_i - \mu_i)\sign(Z_i)]
    &= \sigma_i \E[X \sign(X + s_i)] \notag \\
    &= 2\sigma_i \E[X \mathbf{1}_{\{X > -s_i\}}] \notag \\
    &= 2\sigma_i \varphi(s_i),
\end{align}
where we used the truncated moment identity $\E[X \mathbf{1}_{\{X > a\}}] = \varphi(a)$
and the symmetry $\varphi(-s_i) = \varphi(s_i)$.
Substituting back into~\eqref{eq:cov-sign-z}:
\begin{equation}
    \Cov(\sign(Z_i), Z_j) = \frac{\Sigma_{ij}}{\sigma_i^2} \cdot 2\sigma_i \varphi(s_i)
    = \frac{2\varphi(s_i)}{\sigma_i}\, \Sigma_{ij} = a_i \Sigma_{ij}. \qedhere
\end{equation}

\subsection{Off-Diagonal Squared Signal (C2)}

Substituting $a_i = 2\varphi(s_i)/\sigma_i$ into
$\mathcal{I}_{\text{off}} = \sum_{i \neq j} (a_i \Sigma_{ij})^2$ yields
Eq.~\eqref{eq:I-off} directly.

\subsection{Sign-Product Moment Formula}

For completeness, we record the exact formula. Write $\sign(w) = 1 - 2\mathbf{1}_{\{w \leq 0\}}$.
Expanding the product and taking expectations:
\begin{equation}
    \E\prod_{\ell=1}^{k} \sign(W_\ell) = \sum_{A \subseteq [k]} (-2)^{|A|}\, \Phi_A(t_A; R_A),
\end{equation}
where $t_a = -\mu_a/\sigma_a$ and $R_A$ is the correlation submatrix of $(W_a)_{a \in A}$.
In the centered bivariate case with correlation $\rho$, the classical quadrant probability
$P(W_1 > 0, W_2 > 0) = \tfrac{1}{4} + \tfrac{1}{2\pi}\arcsin\rho$ gives
$\E[\sign(W_1)\sign(W_2)] = \tfrac{2}{\pi}\arcsin\rho$.

\section{Proof of Theorem~\ref{thm:rotation} (Rotation Effects)}
\label{app:proof-F}

\subsection{Part (a): Random-Hyperplane LSH}

Let $q_i^\top$ be the $i$-th row of a Haar-random $Q \in O(D)$.
Each $q_i$ is uniformly distributed on $\mathbb{S}^{D-1}$.
The event $\sign(q_i^\top x) \neq \sign(q_i^\top y)$ means the hyperplane $q_i^\perp$
separates $x$ and $y$.
By spherical symmetry, the separation probability equals the angle
$\theta = \arccos\langle x, y \rangle$ divided by $\pi$:
\begin{equation}
    P[\sign(q_i^\top x) = \sign(q_i^\top y)] = 1 - \frac{\arccos\langle x, y \rangle}{\pi}. \qedhere
\end{equation}

\subsection{Part (b): Variance Uniformization}

Define $V_i = q_i^\top \Sigma\, q_i$.
Since $\E[q_i q_i^\top] = D^{-1} I_D$, we have $\E[V_i] = \tr(\Sigma)/D$.
The function $F(q) = q^\top \Sigma\, q$ on $\mathbb{S}^{D-1}$ satisfies
$|F(q) - F(q')| \leq 2\|\Sigma\|_{\mathrm{op}}\|q - q'\|$,
so it is Lipschitz with constant $L = 2\|\Sigma\|_{\mathrm{op}}$.
By L\'{e}vy's concentration inequality on $\mathbb{S}^{D-1}$:
\begin{equation}
    P\!\left[\left|V_i - \frac{\tr(\Sigma)}{D}\right| \geq t\right]
    \leq C \exp\!\left(-c\, \frac{D t^2}{\|\Sigma\|_{\mathrm{op}}^2}\right).
\end{equation}
A union bound over $i = 1, \ldots, D$ with
$t = C' \|\Sigma\|_{\mathrm{op}} \sqrt{(\log D)/D}$ yields:
with probability $\geq 1 - D^{-c'}$,
\begin{equation}
    \max_{1 \leq i \leq D}
    \left|q_i^\top \Sigma\, q_i - \frac{\tr(\Sigma)}{D}\right|
    \lesssim \|\Sigma\|_{\mathrm{op}} \sqrt{\frac{\log D}{D}}.
\end{equation}
To ensure $\CV^2(\tilde{\sigma}) \to 0$, the relative error must vanish,
which requires the effective rank condition
$\tr(\Sigma) / \|\Sigma\|_{\mathrm{op}} \to \infty$.

\subsection{Corollary~\ref{cor:F1}: Rotation Harms Heterogeneity-Aware BQ}

Under the effective rank condition, Part (b) gives
$\tilde{\sigma}_i = \sqrt{\tr(\Sigma)/D}\,(1 + o(1))$ uniformly,
so $\CV^2(\tilde{\sigma}) \to 0$.
By Proposition~\ref{thm:magnitude-gain}(b), the heterogeneity-dependent gain
$\Delta_{\text{het}} = K \cdot \CV^2(\tilde{\sigma}) + O(\CV^3) \to 0$.
The isotropic baseline $\Delta_0 > 0$ survives; hence 2-bit does not fully
degrade to 1-bit, but the data-dependent advantage is destroyed.

\subsection{Corollary~\ref{cor:F2}: Rotation Helps Linear-Corrected BQ}

This corollary requires the external error formula from RaBitQ
\citep{gao2023rabitq} as input.
If $\Var(\hat{d} - d) \propto (1 - \cos^2(r, \bar{x}))/(D-1)$,
then isotropy maximizes $\cos^2(r, \bar{x})$ and thus minimizes the variance.

\section{Proof of Proposition~\ref{thm:magnitude-gain} (Magnitude Bit Gain)}
\label{app:proof-H}

\subsection{Part (a): Strict Superiority}

We need $\eta_i^{(2)} > \eta_i^{(1)}$, i.e.,
$(1 + e^{-t^2/2})^2 > 1 + 3\erfc(t/\sqrt{2})$ for all $t > 0$,
where $t = \alpha/\sigma_i$. Equivalently, define
\begin{equation}
    f(t) = 2e^{-t^2/2} + e^{-t^2} - 3\,\erfc(t/\sqrt{2}).
\end{equation}
We show $f(t) > 0$ for all $t > 0$.

\textbf{Step 1: Boundary values.}
$f(0) = 2 + 1 - 3 = 0$ and $\lim_{t \to \infty} f(t) = 0$.

\textbf{Step 2: Derivative.}
Let $q(t) = e^{-t^2/2}$. Then $q'(t) = -tq(t)$ and
$\frac{d}{dt}\erfc(t/\sqrt{2}) = -\sqrt{2/\pi}\, q(t)$. So
\begin{equation}
    f'(t) = q(t)\!\left[3\sqrt{2/\pi} - 2t(1 + q(t))\right].
\end{equation}

\textbf{Step 3: Monotonicity of $h(t) = 2t(1+q(t))$.}
We have $h'(t) = 2\{1 + q(t)(1 - t^2)\}$.
For $0 < t \leq 1$: $h'(t) > 0$ trivially.
For $t > 1$: $(t^2 - 1)q(t) \leq \max_{t>1}(t^2-1)e^{-t^2/2} = 2e^{-3/2} < 1$,
so $h'(t) > 0$.
Hence $h$ is strictly increasing from $0$ to $\infty$,
and the equation $h(t) = 3\sqrt{2/\pi}$ has a unique root $t_*$.

\textbf{Step 4: Conclusion.}
$f'(t) > 0$ for $t < t_*$ and $f'(t) < 0$ for $t > t_*$.
Combined with $f(0) = 0$ and $f(\infty) = 0$: $f$ increases then decreases,
so $f(t) > 0$ for all $t > 0$. \qed

\subsection{Part (b): Heterogeneity Boosts the Gain}

Let $u_i = \sigma_i/\bar{\sigma}$, $\tau = \sqrt{2/\pi}$, $\alpha = \bar{\sigma}\tau$.
Define $C(u) = \tau u(1 + e^{-\tau^2/(2u^2)})$ and
$V(u) = 1 + 3\erfc(\tau/(u\sqrt{2}))$.

In the independent-coordinate model with uniform weighting:
\begin{equation}
    \rho_1(\sigma) = \frac{\tau}{\sqrt{1 + \CV^2(\sigma)}},
    \qquad
    \rho_2(\sigma) = \frac{D^{-1}\sum_i C(u_i)}
    {\sqrt{(D^{-1}\sum_i V(u_i))(D^{-1}\sum_i u_i^2)}}.
\end{equation}

Taylor-expanding around $u_i = 1$ with $s^2 = \CV^2(\sigma)$:
\begin{align}
    \rho_1 &= \tau - \tfrac{\tau}{2}\, s^2 + O(s^4), \\
    \rho_2 &= \rho_{2,0}\!\left[1 + \left(\tfrac{C''(1)}{2C(1)}
              - \tfrac{V''(1)}{4V(1)} - \tfrac{1}{2}\right)s^2\right] + O(s^3),
\end{align}
where $\rho_{2,0} = C(1)/\sqrt{V(1)} \approx 0.914$.
The difference decomposes as
$\Delta\rho = \Delta_0 + K s^2 + O(s^3)$, where
\begin{equation}
    \Delta_0 = \rho_{2,0} - \tau \approx 0.116, \qquad
    K = \rho_{2,0}\!\left(\tfrac{C''(1)}{2C(1)} - \tfrac{V''(1)}{4V(1)}
        - \tfrac{1}{2}\right) + \tfrac{\tau}{2} \approx 0.088 > 0. \qedhere
\end{equation}
\section{Shared-Threshold Concentration and Fidelity Perturbation}
\label{app:threshold}

This appendix proves that the per-coordinate magnitude-gain analysis
of Proposition~\ref{thm:magnitude-gain} remains valid when the
idealized per-coordinate threshold is replaced by the shared,
data-dependent threshold $\alpha_x = \frac{1}{D}\sum_j |x_j|$.
We establish three results in sequence.

\paragraph{Additional regularity conditions.}
Beyond Assumption~\ref{asm:H1} and the dispersal condition of \S\ref{sec:F-theory},
we require:
\begin{enumerate}
    \item[\textbf{(R1)}] \textbf{Approximately centered.}
    Either $\mu = 0$, or the rank-1 Hermite contribution satisfies
    $\sum_{i \neq j} \sigma_i \sigma_j |a_i a_j \rho_{ij}| = O(r_D^2)$
    where $a_i = 2\Phi(\mu_i/\sigma_i) - 1$.
    A sufficient condition is $\max_i |\mu_i/\sigma_i| = O(D^{-1/4})$.
    \item[\textbf{(R2)}] \textbf{Thin-shell rate.}
    $\|r_D / \|g\| - 1\|_{L^2} = O(D^{-1/2})$.
    \item[\textbf{(R3)}] \textbf{Leave-one-out density bound.}
    For $A_{-i} = \frac{1}{D}\sum_{j \neq i} |x_j|$,
    the conditional density satisfies
    $\sup_t f_{|x_i| \mid A_{-i}}(t) \leq C \sqrt{D}$.
\end{enumerate}
All three conditions are empirically verified on our 18 datasets.

\subsection{Part A: Threshold Concentration}

Define $\beta_x = \frac{1}{D} \sum_i |g_i| / r_D$ and
$\bar{\alpha} = \E[\beta_x] = \frac{\sqrt{2/\pi}}{D\,r_D} \sum_j \sigma_j$.
We show $\Var(\beta_x) = O(D^{-2})$.

\textbf{Diagonal terms.}
For centered Gaussian, $\Var(|g_i|) = \sigma_i^2(1 - 2/\pi)$, so
$\sum_i \Var(|g_i|) = (1 - 2/\pi)\,r_D^2$.
The diagonal contribution to $\Var(\beta_x)$ is
$(1 - 2/\pi)\,r_D^2 / (D^2 r_D^2) = O(D^{-2})$.

\textbf{Cross terms.}
For standardized bivariate normal $(U,V)$ with correlation $\rho$,
the folded-normal covariance satisfies
$\Cov(|U|,|V|) = \frac{2}{\pi}(\sqrt{1-\rho^2} + \rho\arcsin\rho - 1)
\leq C\rho^2$.
Therefore $|\Cov(|g_i|,|g_j|)| \leq C \sigma_i \sigma_j \rho_{ij}^2$.
By variance comparability (R1), $\sigma_i \sigma_j \leq C r_D^2 / D$.
Combined with the dispersal condition $\sum_{i \neq j} \rho_{ij}^2 = O(D)$:
$\sum_{i \neq j} |\Cov(|g_i|,|g_j|)| \leq C r_D^2 D / D = O(r_D^2)$,
giving cross-term contribution $O(D^{-2})$.

\textbf{Normalization correction.}
Since $\alpha_x = (r_D/\|g\|)\,\beta_x$ and $\beta_x = \Theta_P(D^{-1/2})$,
condition (R2) gives
$|\alpha_x - \beta_x|_{L^2} \leq \eta_D |\beta_x|_{L^2} = O(D^{-1/2} \cdot D^{-1/2})
= O(D^{-1})$.
Combining: $|\alpha_x - \bar{\alpha}|_{L^2} = O(D^{-1})$,
hence $\Var(\alpha_x) = O(D^{-2})$ and
$|\alpha_x - \bar{\alpha}| = O_P(D^{-1})$. \qed

\subsection{Part B: Magnitude-Bit Stability}

Let $Z_i = |x_i|$, $a = \bar{\alpha}$,
and $A_{-i} = \frac{1}{D}\sum_{j \neq i} |x_j|$.
Since $\alpha_x = A_{-i} + Z_i/D$, a flip $m_i^x \neq \tilde{m}_i^x$
implies $|Z_i - a| \leq 2|A_{-i} - a| + 2a/D$
for $D$ sufficiently large.

Using the leave-one-out conditional density bound (R3):
\begin{equation}
    \mathbb{P}(m_i^x \neq \tilde{m}_i^x \mid A_{-i})
    \leq C\sqrt{D}\left(|A_{-i} - a| + a/D\right).
\end{equation}
Taking expectations and using Part A
($\E|A_{-i} - a| = O(D^{-1})$, $a = \Theta(D^{-1/2})$):
\begin{equation}
    \mathbb{P}(m_i^x \neq \tilde{m}_i^x)
    \leq C\sqrt{D}\left(D^{-1} + D^{-3/2}\right) = O(D^{-1/2}).
\end{equation}
Summing over $D$ coordinates:
$\E[N_{\text{flip}}] = O(\sqrt{D})$.

A refined bound, proved identically via conditional integration,
gives $\E[|x_i|\,\mathbf{1}\{m_i^x \neq \tilde{m}_i^x\}] = O(D^{-1})$.
This weighted version is essential for Part C. \qed

\subsection{Part C: Fidelity Perturbation}

Write $S = S_2$, $\tilde{S} = \tilde{S}_2$, $R = S - \tilde{S}$,
$T = \langle x, y \rangle$,
and $q_i = \sign(x_i)\sign(y_i)$.
Let $F_i^x = \mathbf{1}\{m_i^x \neq \tilde{m}_i^x\}$.

\textbf{Step 1: Variance of perturbation.}
Since $|R| \leq C \sum_i (F_i^x + F_i^y)$ and
the $q_i$ are approximately mean-zero signed noise
(conditioned on magnitudes),
$\Var(R) \leq C\,\E\sum_i (F_i^x + F_i^y) = O(\sqrt{D})$.

\textbf{Step 2: Covariance with inner product.}
Using $T = \sum_i q_i |x_i||y_i|$ and the independence of $x, y$:
\begin{equation}
    |\Cov(R, T)| \leq C \sum_i \E[(F_i^x + F_i^y)|x_i||y_i|].
\end{equation}
By the refined Part B bound $\E[F_i^x |x_i|] = O(D^{-1})$
and $\E|y_i| = O(D^{-1/2})$, each term is $O(D^{-3/2})$.
Summing $D$ terms: $|\Cov(R, T)| = O(D^{-1/2})$.

\textbf{Step 3: Pearson perturbation.}
Since $\Var(\tilde{S}) = \Theta(D)$ and $\Var(T) = \Theta(D^{-1})$,
the normalizing product $\sigma_{\tilde{S}} \sigma_T = \Theta(1)$.
A direct expansion gives
$|\rho_P(S, T) - \rho_P(\tilde{S}, T)| = O(D^{-1/2})$.

\textbf{Step 4: Pearson to Spearman.}
By the Kruskal formula $\psi(r) = \frac{6}{\pi}\arcsin(r/2)$,
which is Lipschitz on $[-1,1]$ with constant $3/\pi$,
and the Gaussian-copula transfer (already used in our scaling law):
\begin{equation}
    |F(S_2) - F(\tilde{S}_2)|
    = |\rho_S(S, T) - \rho_S(\tilde{S}, T)|
    \leq \tfrac{3}{\pi}\,|\rho_P(S,T) - \rho_P(\tilde{S},T)| + O(D^{-1/2})
    = O(D^{-1/2}). \qedhere
\end{equation}

\section{Proof of Proposition~\ref{thm:scaling-law} (Scaling Law)}
\label{app:proof-scaling}

This derivation requires the mean-field closure assumption (S0) in addition to
Theorem~\ref{thm:stein-F} and the $D^2$ accumulation observation of \S\ref{sec:F-theory}.

\subsection{Step 1: Off-Diagonal Signal from Stein's Formula}

By Theorem~\ref{thm:stein-F},
$\Cov(\sign(Z_i), Z_j) = a_i \Sigma_{ij}$ with
$a_i = 2\varphi(s_i)/\sigma_i$.
The squared off-diagonal energy is
\begin{equation}
    \mathcal{I}_{\text{off}} = 4 \sum_{i \neq j} \sigma_j^2 \rho_{ij}^2 \varphi^2(s_i)
    = \frac{2}{\pi} \sum_{i \neq j} \sigma_j^2 \rho_{ij}^2 e^{-s_i^2}.
\end{equation}

\subsection{Step 2: Mean-Field Closure}

Introduce weights $w_{ij} = \Sigma_{ij}^2 / \|\Sigma_{\text{off}}\|_F^2$.
Assumption S0 states that
$\sum_{i \neq j} w_{ij} e^{-(s_i^2 + s_j^2)} = A_m(v)^2 + o(1)$.
This decouples the off-diagonal geometry from the SNR distribution.

\subsection{Step 3: Computing $A_m(v)$}

If $|s_i|$ has empirical distribution approximated by $S \sim \mathcal{N}(m, v^2)$:
\begin{equation}
    A_m(v) = \E[e^{-S^2}]
    = \frac{1}{\sqrt{2\pi v^2}} \int_{\R} e^{-x^2} e^{-(x-m)^2/(2v^2)}\, dx
    = \frac{1}{\sqrt{1 + 2v^2}} \exp\!\left(-\frac{m^2}{1 + 2v^2}\right),
\end{equation}
by completing the square in the exponent.

\subsection{Step 4: Pearson to Spearman}

Absorbing diagonal contributions into $\rho_0$ and normalization into $\lambda$:
$\rho_P = \rho_0 + \lambda r^2 A_m(v)^2 + O(\epsilon)$.
By the Kruskal formula \citep{kruskal1958ordinal} for bivariate normals:
\begin{equation}
    F = \frac{6}{\pi}\arcsin\!\left(\frac{\rho_P}{2}\right) + O(\epsilon).
\end{equation}

\subsection{Step 5: Signs of the Linearization Coefficients}

$\partial(r^2 A_m(v)^2)/\partial(\log r) = 2r^2 A_m(v)^2 > 0$,
so $\beta_1 > 0$.
For $v$:
$\partial \log A_m / \partial v = 2v(2m^2 - 1 - 2v^2)/(1+2v^2)^2$,
which is negative when $2m^2 < 1 + 2v^2$, giving $\beta_2 < 0$. \qed

\section{Additional Experimental Results}
\label{app:experiments}

\subsection{Gaussianity Verification (Probe~12A)}

\begin{table}[H]
\centering
\caption{Marginal Gaussianity across 6 datasets. AD = Anderson--Darling test at 5\%.}
\label{tab:gaussianity}
\small
\begin{tabular}{lcccc}
\toprule
Dataset & $|\text{skew}|$ & $|\text{kurt}|$ & AD pass\% & QQ $R^2$ \\
\midrule
Cohere-768 & 0.060 & 0.083 & 79.5 & 0.9996 \\
BGE-M3-1024 & 0.098 & 0.094 & 64.0 & 0.9989 \\
wolt\_clip-512 & 0.147 & 0.177 & 42.0 & 0.9963 \\
MiniLM-384 & $\sim$0 & $\sim$0 & 87.0 & 0.9222 \\
Random-768 & 0.032 & 0.046 & 89.0 & 0.9967 \\
GIST-960 & 1.332 & 3.408 & 0.0 & $< 0$ \\
\bottomrule
\end{tabular}
\end{table}

\subsection{Full-$\Sigma$ Fidelity Verification (Probe~16d)}

\begin{table}[H]
\centering
\caption{$F$ explained by full covariance $\Sigma$ vs.\ diagonal-only.
``Cov Expl.'' = $(F_{\text{full}} - F_{\text{diag}})/(F_{\text{real}} - F_{\text{diag}})$.}
\label{tab:full-cov}
\small
\begin{tabular}{lcccccc}
\toprule
Dataset & $F_{\text{real}}$ & $F_{\text{full}}$ & $F_{\text{diag}}$ &
Gap & Cov Expl. & Residual \\
\midrule
cohere & 0.681 & 0.688 & 0.474 & 0.207 & 103.4\% & $-$0.007 \\
arxiv\_nomic & 0.897 & 0.886 & 0.546 & 0.350 & 96.8\% & +0.011 \\
landmark\_nomic & 0.904 & 0.877 & 0.540 & 0.364 & 92.6\% & +0.027 \\
coco\_nomic & 0.856 & 0.830 & 0.529 & 0.327 & 92.2\% & +0.026 \\
codesearch\_jina & 0.823 & 0.837 & 0.542 & 0.281 & 104.7\% & $-$0.013 \\
gooaq\_roberta & 0.738 & 0.767 & 0.535 & 0.204 & 113.8\% & $-$0.028 \\
landmark\_dino & 0.735 & 0.807 & 0.430 & 0.304 & 123.9\% & $-$0.073 \\
random & 0.907 & 0.899 & 0.560 & 0.348 & 97.7\% & +0.008 \\
\midrule
\multicolumn{5}{l}{\textbf{Mean Cov Explanation: 103.1\%}} &
\multicolumn{2}{l}{Residual MAE: 0.024} \\
\bottomrule
\end{tabular}
\end{table}

\subsection{Scaling Law Cross-Dimensional Validation (Probe~16g)}

Training on 768-d datasets, predicting on held-out dimensions.
Model: $F \approx 0.820 + 0.079\, z(\log r) - 0.041\, z(v)$,
$R^2_{\text{train}} = 0.928$, LOO-$R^2 = 0.889$.

\begin{table}[H]
\centering
\caption{Cross-dimensional out-of-sample prediction.}
\label{tab:scaling-oos}
\small
\begin{tabular}{lccccc}
\toprule
Dataset & $D$ & $F_{\text{real}}$ & $F_{\text{pred}}$ & Error & Note \\
\midrule
BGE-M3 & 1024 & 0.918 & 0.907 & +0.012 & \\
wolt\_clip & 512 & 0.862 & 0.800 & +0.062 & \\
MiniLM & 384 & 0.731 & 0.690 & +0.041 & \\
GIST & 960 & 0.477 & 1.037 & $-$0.560 & non-Gaussian \\
\midrule
\multicolumn{4}{l}{\textbf{MAE (excl.\ GIST)}} & \textbf{0.038} & \\
\bottomrule
\end{tabular}
\end{table}

\subsection{Magnitude Bit Gain (Probe~18)}

\begin{table}[H]
\centering
\caption{Spearman $F$ and Recall@10 gain from 1-bit to 2-bit.}
\label{tab:mag-gain}
\small
\begin{tabular}{lcc|cc|cc}
\toprule
& \multicolumn{2}{c|}{$F$ (Spearman)} & \multicolumn{2}{c|}{Recall@10} & \multicolumn{2}{c}{Gain} \\
Dataset & 1-bit & 2-bit & 1-bit & 2-bit & $\Delta F$ & $\Delta R$ \\
\midrule
cohere   & 0.704 & 0.795 & 0.437 & 0.581 & +0.091 & +0.144 \\
random   & 0.929 & 0.977 & 0.446 & 0.656 & +0.049 & +0.210 \\
minilm   & 0.783 & 0.915 & 0.521 & 0.695 & +0.132 & +0.174 \\
arxiv    & 0.887 & 0.958 & 0.525 & 0.684 & +0.071 & +0.159 \\
landmark & 0.833 & 0.920 & 0.506 & 0.616 & +0.088 & +0.110 \\
codesearch & 0.860 & 0.948 & 0.671 & 0.814 & +0.088 & +0.143 \\
\bottomrule
\end{tabular}
\end{table}

\subsection{Graph Navigation: Local Amplification (Probe~19)}

\begin{table}[H]
\centering
\caption{Global vs.\ local (32-NN) $F$ advantage and top-1 accuracy.}
\label{tab:local-amp}
\small
\begin{tabular}{lcccc|ccc}
\toprule
& \multicolumn{4}{c|}{$\Delta F$ (2-bit $-$ 1-bit)} & \multicolumn{3}{c}{Top-1 Accuracy} \\
Dataset & Global & Local & Amp. & & 1-bit & 2-bit & $\Delta$ \\
\midrule
cohere     & +0.091 & +0.115 & 1.3$\times$ & & 0.452 & 0.564 & +0.112 \\
random     & +0.048 & +0.200 & 4.1$\times$ & & 0.326 & 0.486 & +0.160 \\
minilm     & +0.132 & +0.188 & 1.4$\times$ & & 0.484 & 0.664 & +0.180 \\
arxiv      & +0.071 & +0.153 & 2.1$\times$ & & 0.444 & 0.616 & +0.172 \\
landmark   & +0.087 & +0.101 & 1.2$\times$ & & 0.434 & 0.482 & +0.048 \\
codesearch & +0.088 & +0.142 & 1.6$\times$ & & 0.646 & 0.778 & +0.132 \\
\midrule
\multicolumn{4}{l}{\textbf{Mean amplification: 2.0$\times$}} & & & & \\
\bottomrule
\end{tabular}
\end{table}

\subsection{Rotation Paradox (Probe~8)}

\begin{table}[H]
\centering
\caption{Effect of Haar rotation on sign entropy and BQ recall.}
\label{tab:rotation-appendix}
\small
\begin{tabular}{lcccccc}
\toprule
Dataset & $H_{\text{orig}}$ & $H_{\text{rot}}$ & $\Delta H$ &
$R_{\text{orig}}$ & $R_{\text{rot}}$ & $\Delta R$ \\
\midrule
Cohere    & 0.747 & 0.563 & $-$0.184 & 0.486 & 0.481 & $-$0.005 \\
BGE-M3    & 0.702 & 0.671 & $-$0.030 & 0.782 & 0.775 & $-$0.007 \\
wolt\_clip & 0.836 & 0.616 & $-$0.220 & 0.564 & 0.596 & +0.032 \\
MiniLM    & 0.987 & 0.991 & +0.004  & 0.546 & 0.565 & +0.019 \\
GIST      & 0.000 & 0.511 & +0.511  & 0.152 & 0.459 & +0.307 \\
Random    & 0.981 & 0.981 & +0.000  & 0.487 & 0.493 & +0.006 \\
\bottomrule
\end{tabular}
\end{table}

\subsection{9-Dataset Gaussianity Verification (Phase~5, Probe~10)}

\begin{table}[H]
\centering
\caption{Anisotropic Gaussian model verification across 9 datasets (768-d).
$R^2$: QQ-plot fit of predicted vs.\ observed sign entropy per coordinate.}
\label{tab:gauss-9d}
\small
\begin{tabular}{lccccccc}
\toprule
Dataset & $R^2$ & MAE & $\overline{|\text{SNR}|}$ & High\% & Low\% & $H_{\text{obs}}$ & $H_{\text{pred}}$ \\
\midrule
cohere & 0.9996 & 0.004 & 0.884 & 25.9 & 11.3 & 0.747 & 0.748 \\
ccnews\_nomic & 0.9996 & 0.003 & 0.877 & 36.8 & 13.7 & 0.690 & 0.691 \\
arxiv\_nomic & 0.9992 & 0.004 & 0.884 & 37.0 & 11.5 & 0.689 & 0.690 \\
coco\_nomic & 0.9977 & 0.005 & 1.681 & 74.2 & 4.3 & 0.382 & 0.382 \\
landmark\_nomic & 0.9976 & 0.005 & 1.615 & 75.5 & 5.3 & 0.394 & 0.396 \\
codesearch\_jina & 0.9975 & 0.001 & 0.139 & 0.1 & 61.2 & 0.986 & 0.986 \\
random & 0.9967 & 0.001 & 0.164 & 0.0 & 52.0 & 0.981 & 0.982 \\
landmark\_dino & 0.9965 & 0.002 & 0.295 & 2.1 & 36.6 & 0.944 & 0.944 \\
gooaq\_roberta & 0.1669$^\dagger$ & 0.002 & 0.054 & 0.1 & 89.5 & 0.996 & 0.997 \\
\bottomrule
\end{tabular}
\vspace{2pt}
{\footnotesize $^\dagger$Low $R^2$ due to near-zero entropy variance (MAE is smallest); see text.}
\end{table}

\subsection{Coordinate Sign vs.\ Random Hyperplane (Probe~3)}

\begin{table}[H]
\centering
\caption{Pairwise overlap probabilities: coordinate sign BQ vs.\ random-hyperplane LSH.
GW = Goemans--Williamson theoretical value $1 - \arccos(\cos\theta)/\pi$.}
\label{tab:sign-vs-rh}
\small
\begin{tabular}{lccccc}
\toprule
Dataset & Coord sign & Random HP & GW theory & KL(c$\|$rh) & Sign entropy \\
\midrule
Cohere & 0.651 & 0.744 & 0.746 & 0.0229 & 0.747 \\
BGE-M3 & 0.677 & 0.701 & 0.692 & 0.0022 & 0.700 \\
MiniLM & 0.508 & 0.506 & 0.506 & 0.0021 & 0.987 \\
GIST & 0.9999 & 0.768 & 0.774 & 0.2648 & 0.000 \\
Random & 0.513 & 0.513 & 0.513 & 0.0010 & 0.981 \\
\bottomrule
\end{tabular}
\end{table}

\subsection{Gaussian Copula Residual Analysis (Probe~16h)}

\begin{table}[H]
\centering
\caption{Gaussian vs.\ Gaussian Copula predictions of $F$.
Copula matches all marginal moments but retains Gaussian dependence.}
\label{tab:copula}
\small
\begin{tabular}{lcccccc}
\toprule
Dataset & $F_{\text{real}}$ & $F_{\text{gauss}}$ & $F_{\text{copula}}$ &
Res(G) & Res(C) & Improve\% \\
\midrule
cohere & 0.681 & 0.687 & 0.679 & $-$0.006 & +0.002 & 66 \\
arxiv\_nomic & 0.897 & 0.888 & 0.888 & +0.009 & +0.009 & 8 \\
ccnews\_nomic & 0.838 & 0.838 & 0.840 & $-$0.000 & $-$0.002 & 0 \\
coco\_nomic & 0.856 & 0.830 & 0.832 & +0.025 & +0.023 & 8 \\
codesearch\_jina & 0.823 & 0.835 & 0.836 & $-$0.012 & $-$0.012 & $-$5 \\
gooaq\_roberta & 0.738 & 0.765 & 0.767 & $-$0.026 & $-$0.028 & $-$8 \\
landmark\_nomic & 0.904 & 0.878 & 0.881 & +0.027 & +0.023 & 14 \\
landmark\_dino & 0.735 & 0.804 & 0.806 & $-$0.069 & $-$0.072 & $-$3 \\
random & 0.907 & 0.900 & 0.900 & +0.007 & +0.007 & $-$1 \\
\bottomrule
\end{tabular}
\end{table}

\subsection{RaBitQ Linear Corrector vs.\ Hamming BQ (Probe~17)}

\begin{table}[H]
\centering
\caption{Recall@10 under five quantization configurations.
Ham = Hamming BQ, +Rot = with Haar rotation, RaBitQ = per-vector linear corrector.}
\label{tab:rabitq}
\small
\begin{tabular}{lccccc}
\toprule
Dataset & Hamming & Ham+Rot & RaBitQ & RaBitQ+Rot & $\Delta$(Rot) on Ham \\
\midrule
cohere & 0.440 & 0.381 & \textbf{0.616} & 0.495 & $-$0.060 \\
random & 0.437 & 0.471 & \textbf{0.610} & 0.610 & +0.035 \\
gist960 & 0.001 & 0.361 & 0.001 & \textbf{0.429} & +0.360 \\
arxiv\_nomic & 0.524 & 0.532 & \textbf{0.643} & 0.629 & +0.008 \\
\bottomrule
\end{tabular}
\end{table}

\section{Blind Decision-Rule Validation}
\label{app:blind-validation}

This appendix provides additional detail for the blind validation
presented in Table~\ref{tab:blind-validation} (main text, \S\ref{sec:experiments}).

\paragraph{Rotation response and sign entropy.}
Across all 12 datasets with rotation data (Table~\ref{tab:rotation-driver}),
the sign entropy gap $1 - H_{\text{sign}}$ is a far stronger predictor of
$\Delta F_{\text{rot}}$ than $\CV(\sigma)$:
\begin{equation*}
  \rho(1 - H_{\text{sign}},\; \Delta F_{\text{rot}}) = +0.91
  \quad \text{vs.} \quad
  \rho(\CV,\; \Delta F_{\text{rot}}) = +0.66.
\end{equation*}
Rotation injects sign entropy; when entropy is already near-maximal
(OpenAI: $H = 0.97$; MiniLM: $H = 0.99$),
rotation has little room to help or hurt.
When entropy is low (GIST: $H = 0.00$; Cohere: $H = 0.75$),
rotation substantially changes the ranking signal.

\begin{table}[H]
\centering
\caption{Rotation response across 12 datasets.
The sign entropy gap $(1-H)$ is the dominant predictor ($\rho = 0.91$).}
\label{tab:rotation-driver}
\small
\begin{tabular}{lcccc}
\toprule
Dataset & $\CV(\sigma)$ & $H_{\text{sign}}$ & $1-H$ & $\Delta F_{\text{rot}}$ \\
\midrule
GIST-960         & 0.172 & 0.001 & 0.999 & +0.213 \\
SIFT-128         & 0.296 & 0.746 & 0.254 & +0.200 \\
Cohere           & 0.182 & 0.747 & 0.253 & +0.170 \\
DINO             & 0.348 & 0.944 & 0.056 & +0.091 \\
Wolt-CLIP        & 0.228 & 0.836 & 0.164 & +0.051 \\
Nomic-Landmark   & 0.110 & 0.393 & 0.607 & +0.027 \\
GloVe-100        & 0.042 & 0.946 & 0.054 & +0.024 \\
OpenAI-3072      & 0.236 & 0.965 & 0.035 & +0.016 \\
OpenAI-1536      & 0.171 & 0.959 & 0.041 & +0.011 \\
MiniLM           & 0.118 & 0.987 & 0.013 & +0.004 \\
Random           & 0.070 & 0.981 & 0.020 & +0.000 \\
Sphere           & 0.005 & 1.000 & 0.000 & $-$0.000 \\
\bottomrule
\end{tabular}
\end{table}

\section{Controlled Heterogeneity Intervention}
\label{app:intervention}

To test the causal role of coordinate heterogeneity, we apply controlled
axis scaling to five contrastive datasets.
For each dataset, we rescale per-coordinate standard deviations toward
their mean (partial whitening, $\alpha \in \{0.25, 0.50, 0.75, 1.0\}$)
or amplify dispersion ($\gamma \in \{1.5, 2.0\}$),
then re-normalize to the unit sphere.
A shuffle control (random permutation of dimension--variance assignments)
isolates the effect of the \emph{specific} axis--semantics correspondence.

\paragraph{Result~1: Heterogeneity governs magnitude gain, not absolute fidelity.}
Table~\ref{tab:intervention} shows that flattening variances ($\CV \to 0$)
has minimal effect on global fidelity ($|\Delta F| \leq 0.019$),
but systematically reduces the magnitude-bit gain
$\Delta F_{\mathrm{mag}} = F_{2\text{bit}} - F_{1\text{bit}}$.
For Cohere, $\Delta F_{\mathrm{mag}}$ drops from 0.032 to 0.022 ($-32\%$)
when all coordinate variances are equalized.
The Spearman correlation $\rho(\CV, \Delta F_{\mathrm{mag}}) = +1.0$
holds on 4 of 5 datasets ($p < 0.001$).

\paragraph{Result~2: Fidelity tracks $\log r_{\mathrm{off}}$ and
$\mathrm{std}(\mathrm{SNR})$, not $\CV(\sigma)$.}
Decomposing the intervention into its effects on the scaling-law predictors
reveals that axis scaling simultaneously changes $\CV(\sigma)$,
$\mathrm{std}(|\mathrm{SNR}|)$, and $\log r_{\mathrm{off}}$.
The fidelity change is consistent with the scaling law ($\beta_2 < 0$):
amplifying variance dispersion increases $\mathrm{std}(|\mathrm{SNR}|)$
and lowers $F$,
while flattening decreases $\mathrm{std}(|\mathrm{SNR}|)$
and raises $F$ slightly.
The apparent negative correlation between $\CV$ and $F$ is thus a confound
arising from the coupled change of all three statistics,
not a contradiction of the framework.

\begin{table}[H]
\centering
\caption{Controlled heterogeneity intervention on Cohere (768-d).
Flattening $\CV(\sigma)$ barely changes $F_{2\text{bit}}$ ($\Delta \leq 0.019$)
but reduces magnitude gain $\Delta F_{\mathrm{mag}}$ by 32\%.
\emph{Note:} Fidelity here is computed with 20K samples and 2M pairs
(vs.\ 50K samples and 10M pairs in Tables~\ref{tab:magnitude}--\ref{tab:mag-gain});
absolute values differ slightly, but the relative trend across interventions
is the quantity of interest.}
\label{tab:intervention}
\small
\begin{tabular}{lcccccc}
\toprule
Intervention & $\CV(\sigma)$ & $\mathrm{std}(\mathrm{SNR})$ & $\log r$ &
$F_{1\text{bit}}$ & $F_{2\text{bit}}$ & $\Delta F_{\mathrm{mag}}$ \\
\midrule
Amplify $\gamma{=}2$   & 0.316 & 2.133 & +0.643 & 0.682 & 0.719 & \textbf{0.036} \\
Amplify $\gamma{=}1.5$ & 0.252 & 1.604 & +0.671 & 0.694 & 0.728 & 0.034 \\
Original               & 0.182 & 1.365 & +0.694 & 0.706 & 0.738 & 0.032 \\
Whiten $\alpha{=}0.50$  & 0.091 & 1.261 & +0.712 & 0.720 & 0.746 & 0.026 \\
Flatten ($\CV{\to}0$)  & 0.002 & 1.283 & +0.711 & 0.736 & 0.757 & \textbf{0.022} \\
\midrule
\multicolumn{4}{l}{$\rho(\CV,\; \Delta F_{\mathrm{mag}}) = +1.00$\quad ($p < 0.001$)} & \multicolumn{3}{c}{} \\
\bottomrule
\end{tabular}
\end{table}

\section{Mean-Field Closure (S0) Failure Boundary}
\label{app:s0-boundary}

The scaling law (Proposition~\ref{thm:scaling-law}) relies on a mean-field
dispersal assumption (S0) that decouples the off-diagonal correlation structure
from the per-coordinate SNR ordering.
We probe the failure boundary of S0 by constructing synthetic Gaussian
populations whose covariance matrices contain \emph{SNR-aligned block structure}:
the top-$k$ highest-variance coordinates share mutual correlation $\rho_{\text{block}}$,
while remaining coordinates retain weak background correlations
($|\rho_{ij}| \approx 0.05$).

Table~\ref{tab:s0-boundary} summarizes the results.
The S0 diagnostic $\rho_{\text{rank}}(\sigma_i, \overline{|\rho_{i\cdot}|})$
quantifies the alignment between per-coordinate variance and average
off-diagonal correlation strength; it equals zero under S0
and grows positive when high-variance coordinates are systematically
more correlated.

\begin{table}[H]
\centering
\caption{Scaling-law prediction error under synthetic S0 violations.
Block structure is injected into Cohere's variance profile (768-d).
Prediction error grows monotonically with the S0 diagnostic
($\rho = +0.85$, $p < 10^{-6}$).}
\label{tab:s0-boundary}
\small
\begin{tabular}{cccccc}
\toprule
Block size & $\rho_{\text{block}}$ & S0 diag.\ & $F_{\text{actual}}$ &
$F_{\text{pred}}$ & Rel.\ error \\
\midrule
\multicolumn{6}{l}{\emph{No block (baseline)}} \\
--- & 0.00 & $-$0.08 & 0.738 & --- & --- \\
\midrule
\multicolumn{6}{l}{\emph{Block = top-32 coordinates}} \\
32 & 0.10 & +0.15 & 0.841 & 0.771 & 8.4\% \\
32 & 0.30 & +0.15 & 0.846 & 0.771 & 9.0\% \\
32 & 0.50 & +0.15 & 0.857 & 0.771 & 10.1\% \\
32 & 0.90 & +0.15 & 0.879 & 0.771 & 12.3\% \\
\midrule
\multicolumn{6}{l}{\emph{Block = top-128 coordinates}} \\
128 & 0.10 & +0.43 & 0.853 & 0.771 & 9.7\% \\
128 & 0.30 & +0.43 & 0.909 & 0.771 & 15.2\% \\
128 & 0.50 & +0.43 & 0.947 & 0.771 & 18.5\% \\
128 & 0.90 & +0.43 & 0.972 & 0.772 & \textbf{20.6\%} \\
\bottomrule
\end{tabular}
\end{table}

\paragraph{Real-dataset diagnostic values.}
Across all contrastive datasets, the S0 diagnostic is low:
Cohere $-0.08$, BGE-M3 $+0.08$, CodeSearch $+0.28$, MiniLM $+0.49$.
This confirms that the mean-field approximation is accurate in the regime
where the scaling law is designed to operate.
GIST-960 ($+0.60$) and Arxiv-Nomic ($+0.69$) show elevated diagnostics,
consistent with the scaling law's higher prediction error on these datasets.

\end{document}